\documentclass{llncs}
\usepackage{amsmath}
\usepackage{amssymb}
\usepackage{times}
\usepackage{indentfirst}
\usepackage{graphicx}
\usepackage{epsfig}
\usepackage{epstopdf}
\usepackage[bookmarks=false]{hyperref}

\begin{document}

\title{Relationship between the second type of covering-based rough set and matroid via closure operator}

\author{Yanfang Liu, William Zhu~\thanks{Corresponding author.
E-mail: williamfengzhu@gmail.com (William Zhu)}}
\institute{
Lab of Granular Computing\\
Zhangzhou Normal University, Zhangzhou 363000, China}

%%%%%%%%%%%%%%%%%%%%%%%%%%%%%%%%%%%%%%%%%%%%%%%%%%%%%%%%%%%%%%%%%%%%%%%%%%%%%%%%%%%%%%%%%%%%%%%%%%%%%%%%%%%%%%%%%%%%%%%%%%%%%%
% Enter your name between curly braces

%\date{\today}  % Enter your date or \today between curly braces

\date{\today}          % Enter your date or \today between curly braces
\maketitle

\begin{abstract}
Recently, in order to broad the application and theoretical areas of rough sets and matroids, some authors have combined them from many different viewpoints, such as circuits, rank function, spanning sets and so on.
In this paper, we connect the second type of covering-based rough sets and matroids from the view of closure operators.
On one hand, we establish a closure system through the fixed point family of the second type of covering lower approximation operator, and then construct a closure operator.
For a covering of a universe, the closure operator is a closure one of a matroid if and only if the reduct of the covering is a partition of the universe.
On the other hand, we investigate the sufficient and necessary condition that the second type of covering upper approximation operation is a closure one of a matroid.

\textbf{Keywords:}
Matroid; covering-based rough set; closure operator; lower and upper approximation; indiscernible neighborhood and neighborhood; reduct.
\end{abstract}

%%%%%%%%%%%%%%%%%%%%%%%%%%%%%%%%%%%%
                                                                                                                                                                                                                          %%%%%%%%%%%%%%%%%%%%%%%%%%%%%%%%%%%%

\section{Introduction}
To deal with the vagueness and granularity in information systems, researchers proposed several methods such as rough set theory~\cite{Pawlak82Rough} and fuzzy theory~\cite{Pawlak85Fuzzy}.
The classical rough sets are based on equivalence relations or partitions which are restrictive for many applications, then they have been extended to relation-based rough sets~\cite{Kryszkiewicz98Rough,Kryszkiewicz98Rule,SlowinskiVanderpooten00AGeneralized,Yao98Constructive,ZhuWang06Binary,Yao98Relational} and covering-based rough sets~\cite{BonikowskiBryniarskiSkardowska98extensionsandintentions,Zhu07Topological,ZhuWang03Reduction,ZhuWang06ANew}.

As a generalization of linear algebra and graph theory, matroid theory~\cite{Lai01Matroid,LiuChen94Matroid} was proposed by Whitney.
Matroids have powerful axiomatic systems which provides a well-platform to connect with other theories.
They have combined with classical rough sets~\cite{LiuZhu12characteristicofpartition-circuitmatroid,LiuZhuZhang12Relationshipbetween,TangSheZhu12matroidal,LiLiu12Matroidal,WangZhuZhuMin12matroidalstructure}, relation-based rough sets~\cite{ZhangWangFengFeng11reductionofrough,WangZhuMin11TheVectorially,ZhuWang11Matroidal,ZhuWang11Rough} and covering-based rough sets~\cite{WangMinZhu11Covering,WangZhuZhuMin12quantitative,WangZhu11Matroidal,WangZhuMin11Transversal}.
In this paper, we connect the second type of covering-based rough sets and matroids through closure operators.

On one hand, for a covering of a universe, the fixed point family of the second type of covering lower approximation operator is a closure system if and only if the covering is unary.
We induce a closure operator by the closure system.
When the family of neighborhoods of any element in the universe forms a partition, the closure operator is a closure one of matroid.
Moreover, we prove that the reduct of a covering is a partition if and only if the covering is unary and the family of all the neighborhoods forms a partition.
That is to say, the reduct of a covering is a partition if and only if the closure operator induced by the fixed point family is a closure operator of a matroid.

On the other hand, we investigate the relationship between the second type of covering upper approximation operator and the closure operator of a matroid.
In~\cite{Zhu06PropertiesoftheSecond}, Zhu has studied the properties of the second type of covering-based rough sets.
He gives the sufficient and necessary condition that the second type of covering upper approximation operator satisfies the idempotency.
However, in fact, the condition is just necessary.
Then, we investigate the same issue and provide the right sufficient and necessary condition.
Moreover, for a covering of a universe, the second type of covering upper approximation operator is a closure one of a matroid if and only if the family of all indiscernible neighborhoods of any element of the universe forms a partition.

The rest of this paper is organized as follows: In Section~\ref{S:basicdefinitions}, we recall some basic definitions of the second type of covering-based rough sets and matroids.
Section~\ref{S:lowerapproximationandclosure} establishes a closure system through the second type of covering lower approximation operator, constructs a closure operator and investigate the sufficient necessary condition that the closure operator is a closure one of a matroid.
In Section~\ref{S:upperapproximationandclosure}, we study the sufficient and necessary condition that the second type of covering upper approximation operator is a closure operator of a matroid.
Finally, we conclude this paper in Section~\ref{S:conclusions}.

%%%%%%%%%%%%%%%%%%%%%%%%%%%%%%%%%%%%%%%%%%%%%%%%%%%%%%%%%%%%%%%%%%%%%%%%%%%%%%%
%%%%%%%%%%%%%%%%%%%%%%%%%%%%%%%%%%%%%%%%%%%%%%%%%%%%%%%%%%%%%%%%%%%%%%%%%%%%%%%
\section{Basic definitions}
\label{S:basicdefinitions}
In this section, we recall some basic definitions and results of the second type of covering-based rough sets and matroids.

\subsection{The second type of covering-based rough sets}
As a generalization of classical rough sets, covering-based rough sets are obtained through extending partitions to coverings.

\begin{definition}(Covering~\cite{ZhuWang03Reduction})
\label{D:covering}
Let $U$ be a universe of discourse and $\mathbf{C}$ a family of subsets of $U$.
If none of subsets in $\mathbf{C}$ is empty and $\cup\mathbf{C}=U$, then $\mathbf{C}$ is called a covering of $U$.
\end{definition}

It is clear that a partition of $U$ is certainly a covering of $U$, so the concept of a covering is an extension of the concept of a partition.

In the description of objects, we do not use all the features attributed to those objects.
We limit ourself only to the most essential ones.
The essential features of an object are established by the following definition.

\begin{definition}(Minimal description~\cite{BonikowskiBryniarskiSkardowska98extensionsandintentions})
\label{D:minimaldescription}
Let $\mathbf{C}$ be a covering of $U$ and $x\in U$.
The following family:\\
\centerline{$Md(x)=\{K\in\mathbf{C}:x\in K\wedge (\forall S\in\mathbf{C}\wedge x\in S\wedge S\subseteq K\Rightarrow K=S)\}$}
is called the minimal description of $x$.
\end{definition}

Unary covering is an important concept of covering-based rough sets.

\begin{definition}(Unary covering~\cite{ZhuWang06Relationships})
\label{D:unarycovering}
Let $\mathbf{C}$ be a covering of $U$.
$\mathbf{C}$ is called unary if $|Md(x)|=1$ for all $x\in U$.
\end{definition}

The core concepts of classical rough sets are the lower and upper approximation operators.
Through different forms of the lower and upper approximations based on coverings, many types of covering-based rough sets are put forward.
In this paper, we investigate only the second type of covering-based rough sets.

\begin{definition}(The second type of covering lower and upper approximation operators~\cite{ZhuWang06Relationships})
\label{D:thesecondtypeofcoveringlowerandupperapproximationoperators}
Let $\mathbf{C}$ be a covering of $U$.
For any $X\subseteq U$,\\
\centerline{$SL_{\mathbf{C}}(X)=\cup\{K\in\mathbf{C}:K\subseteq X\}$,}\\
\centerline{$SH_{\mathbf{C}}(X)=\cup\{K\in\mathbf{C}:K\cap X\neq\emptyset\}$.}
We call $SL, SH$ the second type of covering lower, upper approximation operators, respectively.
When the covering is clear, we omit the lowercase $\mathbf{C}$ for the two operators.
\end{definition}

\begin{proposition}(\cite{BonikowskiBryniarskiSkardowska98extensionsandintentions})
\label{P:theconditionofthefixedpoint}
Let $\mathbf{C}$ be a covering of $U$.
$SL(X)=X$ if and only if $X$ is a union of some elements of $\mathbf{C}$.
\end{proposition}

The second type of covering-based rough sets have the following properties.

\begin{proposition}(\cite{ZhuWang03Reduction,Zhu02OnCovering})
\label{P:propertiesoflowerandupper}
Let $\mathbf{C}$ be a covering of $U$.
For any $X, Y\subseteq U$,\\
(1L) $SL(U)=U$\\
(1H) $SH(U)=U$\\
(2L) $SL(\emptyset)=\emptyset$\\
(2H) $SH(\emptyset)=\emptyset$\\
(3L) $SL(X)\subseteq X$\\
(3H) $X\subseteq SH(X)$\\
(4H) $SH(X\cup Y)=SH(X)\cup SH(Y)$\\
(5L) $SL(SL(X))=SL(X)$\\
(6L) $X\subseteq Y\Rightarrow SL(X)\subseteq SL(Y)$\\
(6H) $X\subseteq Y\Rightarrow SH(X)\subseteq SH(Y)$\\
(7LH) $SL(X)\subseteq SH(X)$
\end{proposition}

\subsection{Matroids}
A matroid is a structure that captures and generalizes the notion of linear independence in vector spaces.
In the following definition, we will introduce a matroid from the viewpoint of independent sets.

\begin{definition}(Matroid~\cite{Lai01Matroid})
\label{D:matroid}
A matroid is a pair $M=(U, \mathbf{I})$ consisting a finite universe $U$ and a collection $\mathbf{I}$ of subsets of $U$ called independent sets satisfying the following three properties:\\
(I1) $\emptyset\in\mathbf{I}$;\\
(I2) If $I\in \mathbf{I}$ and $I'\subseteq I$, then $I'\in \mathbf{I}$;\\
(I3) If $I_{1}, I_{2}\in \mathbf{I}$ and $|I_{1}|<|I_{2}|$, then there exists $u\in I_{2}-I_{1}$ such that $I_{1}\cup \{u\}\in \mathbf{I}$, where $|I|$ denotes the cardinality of $I$.
\end{definition}

There are many different but equivalent ways to define a matroid.
In the following, we will generate a matroid in terms of closure operators.

\begin{proposition}(Closure axiom~\cite{Lai01Matroid})
\label{P:closureaxiom}
Let $cl:2^{U}\rightarrow 2^{U}$ be an operator.
Then there exists a matroid $M$ such that $cl=cl_{M}$ iff $cl$ satisfies the following conditions:\\
(CL1) For all $X\subseteq U$, $X\subseteq cl(X)$;\\
(CL2) For all $X, Y\subseteq U$, if $X\subseteq Y$, then $cl(X)\subseteq cl(Y)$;\\
(CL3) For all $X\subseteq U$, $cl(cl(X))=cl(X)$;\\
(CL4) For all $X\subseteq U, x\in U$, if $y\in cl(X\cup\{x\})-cl(X)$, then $x\in cl(X\cup\{y\})$.
\end{proposition}

\section{The second type of lower approximation operator and closure operator}
\label{S:lowerapproximationandclosure}
In this section, we construct a matroid through the second type of lower approximation operator.
First, we introduce the definition of closure systems.

\begin{definition}(Closure system\cite{LiLiu12Matroidal})
\label{D:closuresystem}
Let $\mathbf{F}$ be a family of subsets of $U$.
$\mathbf{F}$ is called a closure system if it satisfies the following conditions:\\
(F1) If $F_{1}, F_{2}\in\mathbf{F}$, then $F_{1}\cap F_{2}\in\mathbf{F}$;\\
(F2) $U\in\mathbf{F}$.
\end{definition}

Any closure system can induce a closure operator.

\begin{proposition}
\label{P:closuresysteminduceclosureoperator}
Let $\mathbf{F}$ be a closure system of $U$.
$cl_{\mathbf{F}}(X)=\cap\{F\in\mathbf{F}:X\subseteq F\}$ is the closure of $X$ with respect to $\mathbf{F}$ and $cl_{\mathbf{F}}$ is called the closure operator induced by $\mathbf{F}$.
$cl_{\mathbf{F}}$ holds the following properties: for all $X, Y\subseteq U$,\\
(CLF1) $X\subseteq cl_{\mathbf{F}}(X)$;\\
(CLF2) If $X\subseteq Y$, then $cl_{\mathbf{F}}(X)\subseteq cl_{\mathbf{F}}(Y)$;\\
(CLF3) $cl_{\mathbf{F}}(cl_{\mathbf{F}}(X))=cl_{\mathbf{F}}(X)$.
\end{proposition}

We see that the closure operator of a matroid is more than the one induced by a closure system a property (CL4) of Proposition~\ref{P:closureaxiom}.
Through the second type of the lower approximation operator, whether we can construct a closure system or not?
In order to solve this issue, we define a family of subsets of a universe through the fixed points of the second type of covering lower approximation in the following definition.

\begin{definition}(The fixed point family of the second type of covering lower approximation)
\label{D:thefixedpointfamily}
Let $\mathbf{C}$ be a covering of $U$.
We define the fixed point family of the second type of covering lower approximation with respect to $\mathbf{C}$ as follows:\\
\centerline{$\mathbf{S}_{\mathbf{C}}=\{X\subseteq U:SL(X)=X\}$.}
When there is no confusion, we omit the subscript $\mathbf{C}$.
\end{definition}

In the rest of this paper, we will call $\mathbf{S}$ the fixed point family for short unless otherwise stated.
A question is put forward: whether the fixed point family with respect to a covering is a closure system or not?
In order to solve this question, we introduce the following lemma.

\begin{lemma}(\cite{Zhu09RelationshipAmong})
\label{L:SLismultiplicative}
Let $\mathbf{C}$ be a covering of $U$.
$\forall X, Y\subseteq U, SL(X\cap Y)=SL(X)\cap SL(Y)$ if and only if $\mathbf{C}$ is unary.
\end{lemma}

The sufficient and necessary condition that the fixed point family forms a closure system is obtained in the following theorem.

\begin{theorem}
\label{T:thefixedpointformsasystemclosure}
Let $\mathbf{C}$ be a covering of $U$.
$\mathbf{S}$ is a closure system if and only if $\mathbf{C}$ is unary.
\end{theorem}

\begin{proof}
According to Definition~\ref{D:thefixedpointfamily}, $\mathbf{S}=\{X\subseteq U:SL(X)=X\}$.
According to Definition~\ref{D:closuresystem}, we need to prove only $\mathbf{S}$ satisfies (F1) and (F2).\\
(F1) For all $X_{1}, X_{2}\in\mathbf{S}$, $SL(X_{1})=X_{1}, SL(X_{2})=X_{2}$.
According to Lemma~\ref{L:SLismultiplicative}, $SL(X_{1}\cap X_{2})=SL(X_{1})\cap SH(X_{2})$ if and only if $\mathbf{C}$ is unary, i.e., $SL(X_{1}\cap X_{2})=X_{1}\cap X_{2}$ if and only if $\mathbf{C}$ is unary.
Therefore, $X_{1}\cap X_{2}\in\mathbf{S}$.\\
(F2) According to (1L) of Proposition~\ref{P:propertiesoflowerandupper}, $SL(U)=U$, i.e., $U\in\mathbf{S}$.
\end{proof}

Any closure system can induce a closure operator.
We induce a closure operator by the fixed point family through the same method.

\begin{proposition}
\label{P:cl_CisinducedbyS}
Let $\mathbf{C}$ be a unary covering of $U$ and $cl_{\mathbf{C}}(X)=\cap\{S\in\mathbf{S}:X\subseteq S\}$.
Then $cl_{\mathbf{C}}$ is the closure operator induced by $\mathbf{S}$, and it satisfies (CLF1), (CLF2) and (CLF3) of Proposition~\ref{P:closuresysteminduceclosureoperator}.
\end{proposition}

\begin{proof}
According to Definition~\ref{D:closuresystem}, Proposition~\ref{P:closuresysteminduceclosureoperator} and Theorem~\ref{T:thefixedpointformsasystemclosure}, it is straightforward.
\end{proof}

Can the closure operator induced by the fixed point family forms the closure operator of matroid?
When the answer is yes, what is the condition satisfied the covering?
In the following, we investigate the condition.
First, we introduce the definition of neighborhood, as one of important concepts in covering-based rough sets.

\begin{definition}(Neighborhood~\cite{Zhu09RelationshipBetween})
\label{D:neighborhood}
Let $\mathbf{C}$ be a covering of $U$ and $x\in U$.
$N_{\mathbf{C}}(x)=\cap\{K\in\mathbf{C}:x\in K\}$ is called the neighborhood of $x$ with respect to $\mathbf{C}$.
When there is no confusion, we omit the subscript $\mathbf{C}$.
\end{definition}

For a unary covering of a universe, we study the closure of any single point set and any subset in the following proposition, respectively.

\begin{proposition}
\label{P:theclosureofanelementisitsneighborhood}
Let $\mathbf{C}$ be a unary covering of $U$ and $cl_{\mathbf{C}}$ the closure operator induced by the fixed point family $\mathbf{S}$.\\
(1) $cl_{\mathbf{C}}(\{x\})=N(x)$ for any $x\in U$.\\
(2) $cl_{\mathbf{C}}(X)=\underset{x\in X}{\cup}N(x)$ for all $X\subseteq U$.
\end{proposition}

\begin{proof}
According to Definition~\ref{D:thefixedpointfamily}, $\mathbf{S}=\{X\subseteq U:SL(X)=X\}$.
According to Proposition~\ref{P:cl_CisinducedbyS}, $cl_{\mathbf{C}}(X)=\cap\{S\in\mathbf{S}:X\subseteq S\}$ for all $X\subseteq U$.
Since $\mathbf{C}$ is unary, according to Definition~\ref{D:unarycovering}, for any $x\in U$, $|Md(x)|=1$.
According to Definition~\ref{D:minimaldescription} and Definition~\ref{D:neighborhood}, $\cap Md(x)=N(x)$.
Hence $N(x)\in\mathbf{C}$.\\
(1) According to Proposition~\ref{P:theconditionofthefixedpoint}, we see $SL(N(x))=N(x)$.
Therefore $cl_{\mathbf{C}}(\{x\})=\cap\{S\in\mathbf{S}:x\in S\}=N(x)$.\\
(2) According to Proposition~\ref{P:theconditionofthefixedpoint}, we see $SL(\underset{x\in X}{\cup}N(x))=\underset{x\in X}{\cup}N(x)$.
Therefore $\underset{x\in X}{\cup}N(x)\in\mathbf{S}$.
Since $X\subseteq \underset{x\in X}{\cup}N(x)$, then $cl_{\mathbf{C}}(X)\subseteq \underset{x\in X}{\cup}N(x)$, i.e., $cl_{\mathbf{C}}(X)\subseteq \underset{x\in X}{\cup}cl_{\mathbf{C}}(\{x\})$.
According to Proposition~\ref{P:closuresysteminduceclosureoperator} and Proposition~\ref{P:cl_CisinducedbyS}, if $X\subseteq Y$, then $cl_{\mathbf{C}}(X)\subseteq cl_{\mathbf{C}}(Y)$.
Then $\underset{x\in X}{\cup}cl_{\mathbf{C}}(\{x\})\subseteq cl_{\mathbf{C}}(X)$.
Hence $cl_{\mathbf{C}}(X)=\underset{x\in X}{\cup}cl_{\mathbf{C}}(\{x\})$, i.e., $cl_{\mathbf{C}}(X)\subseteq \underset{x\in X}{\cup}N(x)$.
\end{proof}

When the closure operator induced by the fixed point family is one of a matroid, the sufficient and necessary condition is obtained.

\begin{theorem}
Let $\mathbf{C}$ be a unary covering of $U$ and $cl_{\mathbf{C}}$ the closure operator induced by the fixed point family $\mathbf{S}$.
$cl_{\mathbf{C}}$ satisfies (CL4) of Proposition~\ref{P:closureaxiom} if and only if $\{N(x):x\in U\}$ is a partition.
\end{theorem}

\begin{proof}
According to Proposition~\ref{P:closureaxiom}, we need to prove only for all $X\subseteq U, x, y\in U, y\in cl_{\mathbf{C}}(X\cup\{x\})-cl_{\mathbf{C}}(X)\Rightarrow x\in cl_{\mathbf{C}}(X\cup\{y\})$ if and only if $\{N(x):x\in U\}$ is a partition.\\
$(\Rightarrow)$: According to (2L) of Proposition~\ref{P:propertiesoflowerandupper}, $SL(\emptyset)=\emptyset$, i.e., $\emptyset\in\mathbf{S}$.
Therefore $cl_{\mathbf{C}}(\emptyset)=\emptyset$.
Suppose $X=\emptyset$, then $y\in cl_{\mathbf{C}}(X\cup\{x\})-cl_{\mathbf{C}}(X)\Rightarrow x\in cl_{\mathbf{C}}(X\cup\{y\})$, i.e., $y\in cl_{\mathbf{C}}(\{x\})\Rightarrow x\in cl_{\mathbf{C}}(\{y\})$.
According to Proposition~\ref{P:theclosureofanelementisitsneighborhood}, we see if $y\in N(x)$, then $x\in N(y)$.
So $\{N(x):x\in U\}$ is a partition.\\
$(\Leftarrow)$: According to Proposition~\ref{P:theclosureofanelementisitsneighborhood}, $cl_{\mathbf{C}}(\{x\})=N(x)$ for any $x\in U$ and $cl_{\mathbf{C}}(X)=\underset{x\in X}{\cup}N(x)$ for all $X\subseteq U$.
If $y\in cl_{\mathbf{C}}(X\cup\{x\})-cl_{\mathbf{C}}(X)$, i.e., $y\in N(x)$, since $\{N(x):x\in U\}$ is a partition, then $x\in N(y)$, i.e., $x\in cl_{\mathbf{C}}(\{y\})\subseteq cl_{\mathbf{C}}(X\cup\{y\})$.
\end{proof}

From the above theorem, for a covering of a universe, we see the closure operator induced by the fixed point family is the closure operator of a matroid when the covering is unary and its neighborhoods of every element of the universe form a partition.
In the following, we will investigate some properties of the covering.
First, we introduce the definition of reducible elements and related results.

\begin{definition}(A reducible element of a covering~\cite{ZhuWang03Reduction})
\label{D:areducibleelement}
Let $\mathbf{C}$ be a covering of $U$ and $K\in\mathbf{C}$.
If $K$ is a union of some elements in $\mathbf{C}-\{K\}$, we say $K$ is reducible in $\mathbf{C}$, otherwise $K$ is irreducible.
\end{definition}

As shown in~\cite{ZhuWang03Reduction}, if all reducible elements are deleted from a covering $\mathbf{C}$, the remainder is still a covering and has no reducible elements.
We call the new covering the reduct of the original covering and denote it as $reduct(\mathbf{C})$.

\begin{lemma}(\cite{Zhu09RelationshipAmong})
\label{L:thereductofunary}
If $\mathbf{C}$ is unary, then $reduct(\mathbf{C})=\{K\in Md(x):x\in U\}$.
\end{lemma}

We can obtain the following result.

\begin{proposition}
Let $\mathbf{C}$ be a covering of $U$.
$\mathbf{C}$ is unary and $\{N(x):x\in U\}$ is a partition if and only if $reduct(\mathbf{C})$ is a partition.
\end{proposition}

\begin{proof}
$(\Rightarrow)$: According to Definition~\ref{D:minimaldescription} and Definition~\ref{D:neighborhood}, we can obtain $N(x)=\cap Md(x)$.
Since $\mathbf{C}$ is unary, then $|Md(x)|=1$ for all $x\in U$.
According to Lemma~\ref{L:thereductofunary}, we see $reduct(\mathbf{C})$ is a partition.\\
$(\Leftarrow)$: Suppose $\mathbf{C}$ is not unary, then there exists $x\in K_{1}, K_{2}$ such that $K_{1}, K_{2}\in Md(x)$.
According to Definition~\ref{D:areducibleelement}, $K_{1}$ and $K_{2}$ are not reducible elements.
Therefore $K_{1}\in reduct(\mathbf{C})$ and $K_{2}\in reduct(\mathbf{C})$, which is contradictory with the condition that $reduct(\mathbf{C})$ is a partition.
Hence $\mathbf{C}$ is unary.
According to Lemma~\ref{L:thereductofunary}, $reduct(\mathbf{C})=\{K\in Md(x):x\in U\}$.
Since $N(x)=\cap Md(x)$, then $reduct(\mathbf{C})=\{N(x):x\in U\}$.
Since $reduct(\mathbf{C})$ is a partition, then $\{N(x):x\in U\}$ is a partition.
\end{proof}

The sufficient and necessary condition that the operator induced by the fixed point family is the closure operator of a matroid can be briefly described in the following theorem.

\begin{theorem}
Let $\mathbf{C}$ be a covering of $U$.
There exists $M$ such that $cl_{M}=cl_{\mathbf{C}}$ if and only if $reduct(\mathbf{C})$ is a partition.
\end{theorem}

\section{The second type of upper approximation operator and closure operator}
\label{S:upperapproximationandclosure}
Generally, properties of upper approximation in covering-based rough sets and ones of the closure operator in topology have a lot of similarity.
In this section, we will study the relationship between the second type of covering upper approximation operator and the closure operator of a matroid.

In~\cite{Zhu06PropertiesoftheSecond}, Zhu has investigated the sufficient and necessary condition of the idempotency of the second type of covering upper approximation operator.

\begin{theorem}(\cite{Zhu06PropertiesoftheSecond})
$SH$ satisfies \\
$SH(SH(X))=SH(X)$\\
if and only if $\mathbf{C}$ satisfies the following property: $\forall K, K_{1}, \cdots, K_{m}\in\mathbf{C}$, if $K_{1}\cap\cdots\cap K_{m}\neq\emptyset$ and $K\cap (K_{1}\cup\cdots\cup K_{m})\neq\emptyset$, then $K\subseteq (K_{1}\cup\cdots\cup K_{m})$.
\end{theorem}

However, the above theorem satisfies only the necessity.
An counterexample is listed to illustrate the sufficiency of the above theorem.

\begin{example}
Let $U=\{a, b, c\}$ and $\mathbf{C}=\{K_{1}, K_{2}\}$ where $K_{1}=\{a, b\}, K_{2}=\{a, c\}$.
Since it does not exist $K\in\mathbf{C}$ such that $K\neq K_{1}$ and $K\neq K_{2}$, then the condition is a tautology.
However, $SH(\{b\})=\{a, b\}, SH(SH(\{a, b\}))=\{a, b, c\}$.
Therefore, for all $X\subseteq U$, $SH(SH(X))=SH(X)$ is not always satisfied.
\end{example}

In the following, we will study the sufficient and necessary condition.
We first introduce a lemma.

\begin{lemma}(\cite{Zhu06PropertiesoftheSecond})
\label{L:SH(SH(X))'ssufficientcondition}
Let $\mathbf{C}$ be a covering of $U$.
If $\{SH(\{x\}):x\in U\}$ is a partition, then $SH(SH(X))=SH(X)$ for all $X\subseteq U$.
\end{lemma}

\begin{theorem}
\label{T:SH(SH(X))'ssufficientandnecessary}
Let $\mathbf{C}$ be a covering of $U$.
For all $X\subseteq U$, $SH(SH(X))=SH(X)$ if and only if $\{SH(\{x\}):x\in U\}$ is a partition.
\end{theorem}

\begin{proof}
$(\Rightarrow)$: Suppose $\{SH(\{x\}):x\in U\}$ is not a partition, then there exists $x\in U$ such that $x\in SH(\{x_{1}\}), x\in SH(\{x_{2}\})$ and $SH(\{x_{1}\})\neq SH(\{x_{2}\})$.
According to Definition~\ref{D:thesecondtypeofcoveringlowerandupperapproximationoperators}, we see there exist $K_{1}, K_{2}\in\mathbf{C}$ such that $\{x, x_{1}\}\subseteq K_{1}, \{x, x_{2}\}\subseteq K_{2}$.
Therefore, $x_{1}\in SH(\{x\}), x_{2}\in SH(\{x\})$.
According to (3H) of Proposition~\ref{P:propertiesoflowerandupper}, we obtain $SH(\{x\})\subseteq SH(SH(\{\{x_{1}\}\})), SH(\{x\})\subseteq SH(SH(x_{2})), SH(\{x_{1}\})\subseteq SH(SH(\{x\}))$ and $SH(\{x_{1}\})\subseteq SH(SH(\{x\}))$.
Since $X\subseteq U$, $SH(SH(X))=SH(X)$, then $SH(\{x\})=SH(\{x_{1}\})$ and $SH(\{x\})=SH(\{x_{2}\})$, i.e., $SH(\{x_{1}\})=SH(\{x_{2})$ which is contradictory with that $SH(\{x_{1}\})\neq SH(\{x_{2}\})$.
Hence, If for all $X\subseteq U$, $SH(SH(X))=SH(X)$, then $\{SH(\{x\}):x\in U\}$ is a partition.\\
$(\Leftarrow)$: According to Lemma~\ref{L:SH(SH(X))'ssufficientcondition}, it is straightforward.
\end{proof}

When the second type of covering upper approximation operator is the closure operator of a matroid, the condition is investigated in the following.

\begin{theorem}
Let $\mathbf{C}$ be a covering of $U$.
$SH$ is the closure operator of a matroid if and only if $\{SH(\{x\}):x\in U\}$ is a partition.
\end{theorem}

\begin{proof}
We need to prove that $SH$ satisfies (CL1), (CL2), (CL3) and (CL4) of Proposition~\ref{P:closureaxiom}.\\
(CL1): According to (3H) of Proposition~\ref{P:propertiesoflowerandupper}, for all $X\subseteq U$, $X\subseteq SH(X)$;\\
(CL2): According to (6H) of Proposition~\ref{P:propertiesoflowerandupper}, if $X\subseteq Y\subseteq U$, then $SH(X)\subseteq SH(Y)$;\\
(CL3): According to Theorem~\ref{T:SH(SH(X))'ssufficientandnecessary}, for all $X\subseteq U$, $SH(SH(X))=SH(X)$ if and only if $\{SH(\{x\}):x\in U\}$ is a partition.\\
(CL4): For all $X\subseteq U, x, y\in U$, suppose $y\in SH(X\cup\{x\})-SH(X)$.
According to (4H) of Proposition~\ref{P:propertiesoflowerandupper}, for all $X, Y\subseteq U, SH(X\cup Y)=SH(X)\cup SH(Y)$.
Therefore, $y\in SH(X\cup\{x\})-SH(X)=SH(\{x\})-SH(X)\subseteq SH(\{x\})$.
So there exists $K\in\mathbf{C}$ such that $\{x, y\}\subseteq K$.
According to Definition~\ref{D:thesecondtypeofcoveringlowerandupperapproximationoperators}, $x\in SH(\{y\})$.
Since $SH(\{y\})\subseteq SH(X\cup\{y\})$, then $x\in SH(X\cup\{y\})$.
\end{proof}

In order to further depict the second type of covering upper approximation operator, we introduce indiscernible neighborhood in the following definition.

\begin{definition}(Indiscernible neighborhood~\cite{Zhu09RelationshipBetween})
Let $\mathbf{C}$ be a covering of $U$ and $x\in U$.
$I_{\mathbf{C}}(x)=\cup\{K\in\mathbf{C}:x\in K\}$ is called the indiscernible neighborhood of $x$ with respect to $\mathbf{C}$.
When there is no confusion, we omit the subscript $\mathbf{C}$.
\end{definition}

According to the above definition and Definition~\ref{D:thesecondtypeofcoveringlowerandupperapproximationoperators}, for a covering of a universe, the second type of covering upper approximation of any signal point set is equal to the indiscernible neighborhood of the point.

\begin{lemma}
Let $\mathbf{C}$ be a covering of $U$.
$\{SH(\{x\}):x\in U\}$ is a partition if and only if $\{I(x):x\in U\}$ is a partition.
\end{lemma}

We can easily obtain the following theorem.

\begin{theorem}
Let $\mathbf{C}$ be a covering of $U$.
$SH$ is the closure operator of a matroid if and only if $\{I(x):x\in U\}$ is a partition.
\end{theorem}

%%%%%%%%%%%%%%%%%%%%%%%%%%%%%%%%%%%%%%%%%%%%%%%%%%%%%%%%%%%%%%%%%%%%%%%%%%%%%%%
%%%%%%%%%%%%%%%%%%%%%%%%%%%%%%%%%%%%%%%%%%%%%%%%%%%%%%%%%%%%%%%%%%%%%%%%%%%%%%%

%%%%%%%%%%%%%%%%%%%%%%%%%%%%%%%%%%%%%%%%%%%%%%%%%%%%%%%%%%%%%%%%%%%%%%%%%%%%%%%
%%%%%%%%%%%%%%%%%%%%%%%%%%%%%%%%%%%%%%%%%%%%%%%%%%%%%%%%%%%%%%%%%%%%%%%%%%%%%%%
\section{Conclusions}
\label{S:conclusions}
In this paper, we investigated the relationship between the second type of covering-based rough sets and matroids via closure operators.
First, for a covering of a universe, we constructed a closure system through the second type of covering lower approximation operator, and then obtained a closure operator.
The closure operator was the closure one of a matroid if and only if the reduct of the covering was a partition.
Second, the second type of covering upper approximation operator is the closure operator of a matroid if and only if the family of all indiscernible neighborhoods of any element forms a partition.
In future works, we will investigate relationships between other types of covering-based rough sets and matroids via closure operators.

%%%%%%%%%%%%%%%%%%%%%%%%%%%%%%%%%%%%%%%%%%%%%%%%%%%%%%%%%%%%%%%%%%%%%%%%%%%%%%%
%%%%%%%%%%%%%%%%%%%%%%%%%%%%%%%%%%%%%%%%%%%%%%%%%%%%%%%%%%%%%%%%%%%%%%%%%%%%%%%
\section*{Acknowledgments}
This work is supported in part by the National Natural Science Foundation of China under Grant No. 61170128, the Natural Science Foundation of Fujian Province, China, under Grant Nos. 2011J01374 and 2012J01294, the Science and Technology Key Project of Fujian Province, China, under Grant No. 2012H0043 and State key laboratory of management and control for complex systems open project under Grant No. 20110106.
%%%%%%%%%%%%%%%%%%%%%%%%%%%%%%%%%%%%%%%%%%%%%%%%%%%%%%%%%%%%%%%%%%%%%%%%%%%%%%%
%%%%%%%%%%%%%%%%%%%%%%%%%%%%%%%%%%%%%%%%%%%%%%%%%%%%%%%%%%%%%%%%%%%%%%%%%%%%%%%

%\bibliographystyle{splncs}
%\bibliography{E:/liu/matroid/bib/Roughsets1}

\begin{thebibliography}{10}

\bibitem{Pawlak82Rough}
Pawlak, Z.:
\newblock Rough sets.
\newblock International Journal of Computer and Information Sciences
  \textbf{11} (1982)  341--356

\bibitem{Pawlak85Fuzzy}
Pawlak, Z.:
\newblock Fuzzy sets and rough sets.
\newblock Fuzzy Sets and Systems \textbf{17} (1985)  99--102

\bibitem{Kryszkiewicz98Rough}
Kryszkiewicz, M.:
\newblock Rough set approach to incomplete information systems.
\newblock Information Sciences \textbf{112} (1998)  39--49

\bibitem{Kryszkiewicz98Rule}
Kryszkiewicz, M.:
\newblock Rules in incomplete information systems.
\newblock Information Sciences \textbf{113} (1998)  271--292

\bibitem{SlowinskiVanderpooten00AGeneralized}
Slowinski, R., Vanderpooten, D.:
\newblock A generalized definition of rough approximations based on similarity.
\newblock IEEE Transactions on Knowledge and Data Engineering \textbf{12}
  (2000)  331--336

\bibitem{Yao98Constructive}
Yao, Y.:
\newblock Constructive and algebraic methods of theory of rough sets.
\newblock Information Sciences \textbf{109} (1998)  21--47

\bibitem{ZhuWang06Binary}
Zhu, W., Wang, F.:
\newblock Binary relation based rough set.
\newblock In: Fuzzy Systems and Knowledge Discovery. Volume 4223 of LNAI.
  (2006)  276--285

\bibitem{Yao98Relational}
Yao, Y.:
\newblock Relational interpretations of neighborhood operators and rough set
  approximation operators.
\newblock Information Sciences \textbf{111} (1998)  239--259

\bibitem{BonikowskiBryniarskiSkardowska98extensionsandintentions}
Z.Bonikowski, Bryniarski, E., W.Skardowska, U.:
\newblock Extensions and intentions in the rough set theory.
\newblock Information Sciences \textbf{107} (1998)  149--167

\bibitem{Zhu07Topological}
Zhu, W.:
\newblock Topological approaches to covering rough sets.
\newblock Information Sciences \textbf{177} (2007)  1499--1508

\bibitem{ZhuWang03Reduction}
Zhu, W., Wang, F.:
\newblock Reduction and axiomization of covering generalized rough sets.
\newblock Information Sciences \textbf{152} (2003)  217--230

\bibitem{ZhuWang06ANew}
Zhu, W., Wang, F.:
\newblock A new type of covering rough sets.
\newblock In: IEEE International Conference on Intelligent Systems 2006,
  London, 4-6 September. (2006)  444--449

\bibitem{Lai01Matroid}
Lai, H.:
\newblock Matroid theory.
\newblock Higher Education Press, Beijing (2001)

\bibitem{LiuChen94Matroid}
Liu, G., Chen, Q.:
\newblock Matroid.
\newblock National University of Defence Technology Press, Changsha (1994)

\bibitem{LiuZhu12characteristicofpartition-circuitmatroid}
Liu, Y., Zhu, W.:
\newblock Characteristic of partition-circuit matroid through approximation
  number.
\newblock In: Granular Computing. (2012)  376--381

\bibitem{LiuZhuZhang12Relationshipbetween}
Liu, Y., Zhu, W., Zhang, Y.:
\newblock Relationship between partition matroid and rough set through k-rank
  matroid.
\newblock Journal of Information and Computational Science \textbf{8} (2012)
  2151--2163

\bibitem{TangSheZhu12matroidal}
Tang, J., She, K., Zhu, W.:
\newblock Matroidal structure of rough sets from the viewpoint of graph theory.
\newblock to appear in Journal of Applied Mathematics (2012)

\bibitem{LiLiu12Matroidal}
Li, X., Liu, S.:
\newblock Matroidal approaches to rough set theory via closure operators.
\newblock to appear in International Journal of Approximate Reasoning (2012)
  doi:10.1016/j.ijar.2011.12.005

\bibitem{WangZhuZhuMin12matroidalstructure}
Wang, S., Zhu, Q., Zhu, W., Min, F.:
\newblock Matroidal structure of rough sets and its characterization to
  attribute reduction.
\newblock to appear in Knowledge-Based Systems (2012)

\bibitem{ZhangWangFengFeng11reductionofrough}
Zhang, S., Wang, X., Feng, T., Feng, L.:
\newblock Reduction of rough approximation space based on matroid.
\newblock International Conference on Machine Learning and Cybernetics
  \textbf{2} (2011)  267--272

\bibitem{WangZhuMin11TheVectorially}
Wang, S., Zhu, W., Fan, M.:
\newblock The vectorially matroidal structure of generalized rough sets based
  on relations.
\newblock In: Granular Computing. (2011)  708--711

\bibitem{ZhuWang11Matroidal}
Zhu, W., Wang, S.:
\newblock Matroidal approaches to generalized rough sets based on relations.
\newblock International Journal of Machine Learning and Cybernetics \textbf{2}
  (2011)  273--279

\bibitem{ZhuWang11Rough}
Zhu, W., Wang, S.:
\newblock Rough matroid.
\newblock In: Granular Computing. (2011)  817--822

\bibitem{WangMinZhu11Covering}
Wang, S., Min, F., Zhu, W.:
\newblock Covering nunbers in covering-based rough sets.
\newblock In: Rough Sets, Fuzzy Sets, Data Mining and Granular Computing.
  (2011)  72--78

\bibitem{WangZhuZhuMin12quantitative}
Wang, S., Zhu, Q., Zhu, W., Min, F.:
\newblock Quantitative analysis for covering-based rough sets using the upper
  approximation number.
\newblock to appear in Information Sciences (2012)

\bibitem{WangZhu11Matroidal}
Wang, S., Zhu, W.:
\newblock Matroidal structure of covering-based rough sets through the upper
  approximation number.
\newblock International Journal of Granular Computing, Rough Sets and
  Intelligent Systems \textbf{2} (2011)  141--148

\bibitem{WangZhuMin11Transversal}
Wang, S., Zhu, W., Min, F.:
\newblock Transversal and function matroidal structures of covering-based rough
  sets.
\newblock In: Rough Sets and Knowledge Technology. (2011)  146--155

\bibitem{Zhu06PropertiesoftheSecond}
Zhu, W.:
\newblock Properties of the second type of covering-based rough sets.
\newblock In: Workshop Proceedings of GrC\&BI 06, IEEE WI 06, Hong Kong, China,
  December 18. (2006)  494--497

\bibitem{ZhuWang06Relationships}
Zhu, W., Wang, F.:
\newblock Relationships among three types of covering rough sets.
\newblock In: Granular Computing. (2006)  43--48

\bibitem{Zhu02OnCovering}
Zhu, F.:
\newblock On covering generalized rough sets.
\newblock Master's thesis, The University of Arizona, Tucson, Arizona, USA
  (2002)

\bibitem{Zhu09RelationshipAmong}
Zhu, W.:
\newblock Relationship among basic concepts in covering-based rough sets.
\newblock Information Sciences \textbf{179} (2009)  2478--2486

\bibitem{Zhu09RelationshipBetween}
Zhu, W.:
\newblock Relationship between generalized rough sets based on binary relation
  and covering.
\newblock Information Sciences \textbf{179} (2009)  210--225

\end{thebibliography}

% Set the ending of a LaTeX document
\end{document}